\RequirePackage{iftex}
\ifPDFTeX
  \pdfoutput=1
\fi
\documentclass[conference]{IEEEtran}
\usepackage{graphicx}
\usepackage{booktabs}
\usepackage{amsmath,amssymb,amsthm}
\usepackage{multirow}
\usepackage[hidelinks]{hyperref}
\usepackage{xcolor}
\usepackage{cite}

\newtheorem{proposition}{Proposition}
\newcommand{\method}{FreezePlan}
\newcommand{\xh}{\hat{x}_0}
\newcommand{\tsb}[1]{\textbf{#1}}

\begin{document}

\title{Freeze What Is Settled, Cache What Is Smooth:\\Spatial Freezing Keeps Step Caching in Its Low-Error Regime}

\author{\IEEEauthorblockN{Baiyan Li}
\IEEEauthorblockA{Lancaster University\\
Beijing Jiaotong University\\
libaiyan@bjtu.edu.cn}
\and
\IEEEauthorblockN{Ganxuan Yang}
\IEEEauthorblockA{Lancaster University\\
Beijing Jiaotong University\\
23722069@bjtu.edu.cn}}

\maketitle

\begin{abstract}
Instruction-based image editors built on diffusion transformers run tens of denoising steps.
Training-free acceleration either reuses features across steps (step caching) or stops computing the regions an edit leaves untouched (region freezing), and the two have been combined before.
We ask how a high speed-up target should be split between them.
On Qwen-Image-2.1, freezing does more than remove tokens: it lowers the cost of every step, so the step cache can reach the same overall speed-up with a conservative threshold that keeps its error small.
A per-step cost model fitted on development runs predicts held-out latency with a median error of 4.0\% and reproduces the speed ceilings of the region methods we re-implemented (2.73$\times$ predicted vs.\ 2.70$\times$ measured for RegionE).
Driving frozen tokens to the model's one-step $x_0$ estimate beats returning them to the reference image on every development sample, and in 90\% of measured cases the composed error stays below the sum of the two individual errors.
The resulting method, \method{}, keeps five full steps and gives the rest of the budget to TaylorSeer.
On a pre-registered test split it runs 5.05$\times$ faster end to end at 29.8\,dB PSNR to the original output, with no significant drop in a VLM judge score; TaylorSeer at 5.02$\times$ reaches 23.5\,dB.
The curve extends to 5.99$\times$ at 28.0\,dB.
Swapping in the region selectors of prior methods under our schedule lowers median PSNR by only 0.5--0.6\,dB, which places the gain in the schedule and budget split rather than in the selector.
\end{abstract}


\section{Introduction}
\label{sec:intro}

Diffusion-transformer editors such as FLUX.1 Kontext~\cite{kontext}, Qwen-Image~\cite{qwenimage} and Qwen-Image-2.1~\cite{qwenimage21} follow free-form instructions with high fidelity, but every edit runs tens of denoising steps.
On one RTX~4090, a 1024-pixel edit with Qwen-Image-2.1 takes 24.5\,s over 40 steps, and a full transformer step costs about 0.58\,s.
Training-free acceleration follows two routes.
\emph{Step caching}~\cite{deepcache,fora,deltadit,teacache,toca,duca,taylorseer,foca,freqca,hyca,cachedit} reuses or extrapolates features across steps.
\emph{Region freezing}~\cite{spotedit,regione} exploits the fact that an edit usually changes a small part of the image and stops computing the rest.

Editing has both spatial and temporal redundancy, and prior work already exploits the two together.
EEdit~\cite{eedit} combines spatial locality with step skipping for inversion-based editing, RegionE~\cite{regione} couples region partition, region K/V caching and velocity caching, and SpotEdit~\cite{spotedit} stacks its region skipping with TeaCache and TaylorSeer.
Budget-constrained scheduling has precedents as well: BudCache~\cite{budcache} searches step-level caching policies under a fixed number of function evaluations, and ProCache~\cite{procache} searches constrained caching schedules with selective token computation.
None of these works explains how the errors and costs of freezing and caching \emph{interact}, and that interaction decides how a high speed-up target should be split between them.

On Qwen-Image-2.1 both routes fall short beyond about 4$\times$ (Table~\ref{tab:main}).
TaylorSeer drops from 32.8\,dB at 4.1$\times$ to 23.5\,dB at 5.0$\times$ (SSIM 0.97$\to$0.71), and the region methods we re-implemented with their original schedules stop at 2.70$\times$ (RegionE) and 3.90$\times$ (SpotEdit) even with TaylorSeer added.
Other works report higher speed-ups on other models, e.g., 6.24$\times$ for HyCa~\cite{hyca} on Qwen-Image and up to 10$\times$ for SpecEdit~\cite{specedit}, under different quality protocols or acceleration mechanisms.

We treat high-speed-up editing as a computation-budget split under a fidelity constraint and make three observations.
(i)~Under Euler sampling, a frozen token driven toward a fixed target reaches it exactly (Prop.~\ref{prop:anchor}); its error is set by the target, and the model's one-step $x_0$ estimate is a better target than the reference image.
(ii)~A linear per-step cost model explains real latency; the number of full steps sets the speed ceiling, and freezing lowers the cost of every other step.
(iii)~For the same overall speed-up, the step cache can therefore use a conservative threshold, and in 90\% of measured cases the composed error stays below the sum of the individual errors.
Our contributions are:
\begin{itemize}
  \item A per-step cost model, validated out of sample, that explains the speed ceilings of the region methods we re-implemented and motivates spending few full steps and a conservative cache (Sec.~\ref{sec:cost}).
  \item Measurements of frozen-token targets and composed errors (Secs.~\ref{sec:anchor},~\ref{sec:compose}).
  \item \method{} (Fig.~\ref{fig:method}), which computes only active tokens on the native prefix KV cache of Qwen-Image-2.1 and outperforms every tested method that reaches 5--6$\times$ on a pre-registered test split. A schedule-matched comparison attributes the gain to the schedule rather than to the region selector (Sec.~\ref{sec:exp}).
\end{itemize}

\section{Related Work}
\label{sec:related}

\textbf{Step caching.}
DeepCache~\cite{deepcache}, FORA~\cite{fora} and $\Delta$-DiT~\cite{deltadit} reuse features of neighbouring steps; TeaCache~\cite{teacache} and DBCache~\cite{cachedit} decide adaptively when to reuse; ToCa~\cite{toca} and DuCa~\cite{duca} select tokens; TaylorSeer~\cite{taylorseer}, FoCa~\cite{foca} and FreqCa~\cite{freqca} forecast instead of copying; HyCa~\cite{hyca} assigns different solvers to different feature dimensions.
These methods improve the cache operator.
We leave the operator unchanged and reduce the share of the speed-up it must deliver, so stronger operators are complementary.

\textbf{Budget-constrained schedules.}
BudCache~\cite{budcache} and ProCache~\cite{procache} allocate a budget along the step axis.
We study how a second axis, spatial freezing, moves the operating point of the step axis.

\textbf{Region-aware editing.}
SpotEdit~\cite{spotedit} selects regions with a perceptual score, returns frozen regions to the reference and re-decides at fixed steps.
RegionE~\cite{regione} partitions with a one-step $x_0$ estimate, drives frozen regions to that estimate, keeps their K/V and caches velocities of the edited region.
EEdit~\cite{eedit} targets inversion-based editing; SpecEdit~\cite{specedit} uses dynamic resolution.
Our frozen-token target and K/V retention follow RegionE; we differ in the analysis and in the schedule it leads to.
RAS~\cite{ras}, HSA~\cite{hsa} and HetCache~\cite{hetcache} likewise keep K/V of inactive tokens in generation.

\textbf{Fewer steps and distillation.}
Fewer steps or step distillation~\cite{lcm,add,dmd} change the sampling trajectory and hence the output; we compare against fewer steps and discuss the difference in Sec.~\ref{sec:fewsteps}.

\section{Analysis and Method}
\label{sec:method}

\begin{figure*}[t]
  \centering
  \includegraphics[width=0.86\textwidth]{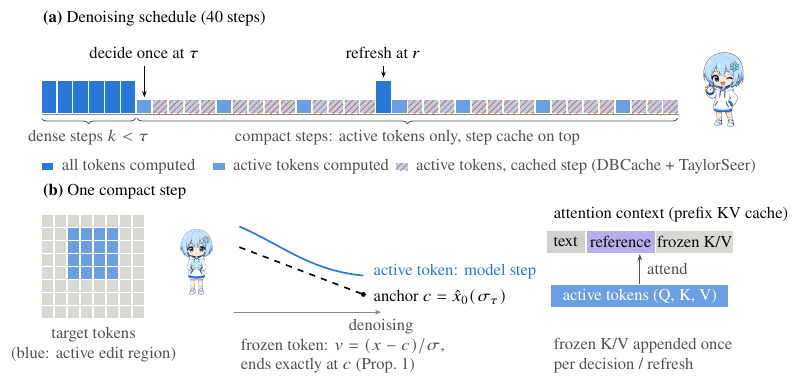}
  \caption{\method{} (Sec.~\ref{sec:freezeplan}). (a) Schedule: $\tau$ full steps, one region decision, compact steps over active tokens with step caching, one optional refresh. (b) In a compact step, frozen tokens move straight to their anchor $c=\xh(\sigma_\tau)$ and their K/V is appended once to the prefix cache that active tokens attend to. The cached-step pattern in (a) is illustrative.}
  \label{fig:method}
\end{figure*}

\subsection{Preliminaries}
Rectified-flow models~\cite{flowmatching,rectflow} write $x_\sigma=(1-\sigma)x_0+\sigma\epsilon$ and predict the velocity $v_\theta(x_\sigma,\sigma)\approx\epsilon-x_0$.
Euler sampling over $\sigma_0=1>\dots>\sigma_K=0$ ($K=40$) updates $x_{k+1}=x_k+(\sigma_{k+1}-\sigma_k)\,v_\theta(x_k,\sigma_k)$, and the one-step estimate is $\xh(\sigma_k)=x_k-\sigma_k v_\theta$.
We denote the 40-step output by $x_0^\ast$.
Qwen-Image-2.1 is block-causal: the instruction and the reference image are encoded once into a prefix KV cache, and each step processes only the target tokens.
A region method splits target tokens into an active set $A$ (fraction $\rho$) and a frozen set $F$.

\subsection{Frozen tokens reach their anchor}
\label{sec:anchor}
\begin{proposition}\label{prop:anchor}
Let token $i$ be frozen from step $k_0$ with a fixed anchor $c_i$ and velocity $v=(x_k-c_i)/\sigma_k$. Then $x_k-c_i=(\sigma_k/\sigma_{k_0})(x_{k_0}-c_i)$ for $k\ge k_0$, so $x_K=c_i$.
\end{proposition}
\noindent\emph{Proof.} $x_{k+1}-c_i=(x_k-c_i)\,(1+(\sigma_{k+1}-\sigma_k)/\sigma_k)=(\sigma_{k+1}/\sigma_k)(x_k-c_i)$; telescope and use $\sigma_K=0$. \hfill$\square$

The statement concerns latents and a fixed anchor.
While frozen, a token's trajectory depends only on its start, its anchor and the noise schedule, not on how active tokens are computed or cached.
At an intermediate refresh step $r$ ($\sigma_r\neq0$) the token has not reached $c_i$; \method{} then sets a new anchor $\xh(\sigma_r)$ and holds it to the end, so $x_K=\xh(\sigma_r)$, which does depend on the cached trajectory before $r$.
Frozen K/V still influence active tokens, and VAE decoding mixes neighbouring latents.

\textbf{Choosing the anchor.}
Returning frozen tokens to the reference (SpotEdit) leaves an error $x_0^\ast-\mathrm{ref}$ that does not shrink with the decision step, because the dense model itself slightly re-renders unedited regions.
Driving them to $\xh(\sigma_\tau)$ (RegionE) leaves $\xh(\sigma_\tau)-x_0^\ast$, which shrinks as $\tau$ grows.
On the truly unchanged tokens of 44 development samples (Fig.~\ref{fig:anchor}), the $x_0$ anchor is better on every sample and every decision step; at $\tau=4$ it reduces latent error by 5.2\,dB and raises the decoded PSNR of the unchanged region to the dense output from 29.0 to 32.7\,dB.
Taking the estimate later (a refresh) helps further (34.7\,dB at step 16).
Mixing the two anchors per frequency band gains at most 0.7\,dB, and extrapolating the $x_0$ trajectory in $\sigma$ is 4--56\,dB worse because the trajectory first drifts and then converges.
The price is fidelity to the input in unchanged regions (29.9 vs.\ 33.3\,dB for the reference anchor), which equals that of the dense model itself (29.7\,dB).

\begin{figure}[t]
  \centering
  \includegraphics[width=\columnwidth]{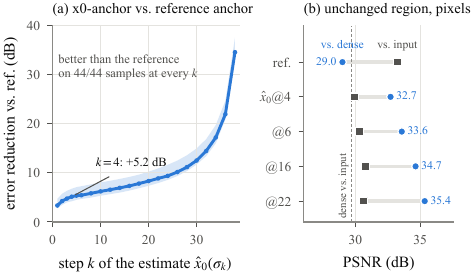}
  \caption{Frozen-token anchor (development set, unchanged tokens). (a) Latent error reduction of the $x_0$ anchor over the reference anchor vs.\ the step of the estimate (median and IQR over 44 samples). (b) Decoded PSNR of the unchanged region to the dense output and to the input.}
  \label{fig:anchor}
\end{figure}

\subsection{A per-step cost model}
\label{sec:cost}
We model the end-to-end time of one edit as
\begin{equation}
T=T_\mathrm{fix}+n_\mathrm{full}\,C+n_\mathrm{comp}\,(\bar\rho\,C+o)+n_\mathrm{cache}\,c,
\label{eq:cost}
\end{equation}
where $n_\mathrm{full}$, $n_\mathrm{comp}$ and $n_\mathrm{cache}$ count full steps, compact steps over active tokens and cached steps, $\bar\rho$ is the mean active fraction, and $T_\mathrm{fix}$ covers text encoding and VAE.
A least-squares fit on 968 development runs (22 configurations $\times$ 44 samples) gives $T_\mathrm{fix}=1.48$\,s, $C=0.575$\,s, $o=0.005$\,s and $c=0.012$\,s ($R^2=0.997$).
Without refitting, the model predicts the 1056 runs of the test split, including configurations absent from the development set (fewer steps, SpotEdit+TS, RegionE+TS), with a median absolute time error of 4.0\% (90th percentile 7.9\%).

Setting every non-full step to zero cost gives a speed ceiling that depends only on $n_\mathrm{full}$ (Table~\ref{tab:caps}), in the spirit of Amdahl's law~\cite{amdahl}.
RegionE's schedule (9 warm-up, 1 forced refresh, 3 smoothing steps) has 13 full steps and a predicted ceiling of 2.73$\times$; with TaylorSeer we measure 2.70$\times$.
SpotEdit keeps 7 full steps and also decodes for its perceptual selector; its ceiling is 4.45$\times$ and we measure 3.90$\times$.
The limit comes from the original schedules rather than from region freezing itself; under our schedule the same selectors exceed 5$\times$ (Sec.~\ref{sec:sched}).

\begin{table}[t]
  \centering
  \caption{Speed ceilings implied by Eq.~\eqref{eq:cost} (all non-full steps free) and the highest measured speed-up on the test split.}
  \label{tab:caps}
  \setlength{\tabcolsep}{4pt}
  \begin{tabular}{lccc}
    \toprule
    Schedule & Full steps & Ceiling & Measured max.\\
    \midrule
    RegionE~\cite{regione} & 13 & 2.73$\times$ & 2.70$\times$ (+TS)\\
    SpotEdit~\cite{spotedit} & 7 & 4.45$\times$ & 3.90$\times$ (+TS)\\
    \method{} & 5 & 5.62$\times$ & 5.48$\times$\\
    \method{} (no refresh) & 4 & 6.48$\times$ & 5.99$\times$\\
    \bottomrule
  \end{tabular}
\end{table}

\subsection{Composed errors and the budget split}
\label{sec:compose}
Freezing lowers the cost of compact and cached steps, so a given threshold $\delta$ yields a higher speed-up, and a given speed-up needs a smaller $\delta$.
On the development set, reaching 5$\times$ requires $\delta=0.8$ for TaylorSeer alone (24.0\,dB) but only $\delta=0.3$ when composed with freezing (31.0\,dB; Fig.~\ref{fig:mech}b).
This helps only if the two errors do not amplify each other.
For 30 development samples in which tokens were frozen, two refresh settings and three thresholds, we measured the MSE of freezing alone, caching alone (same $\delta$) and the composition.
In 90\% of the 180 cases the composed MSE is below the sum of the individual ones (95\% CI 86--94\%, bootstrapped over samples); the median ratio is 0.76 (CI 0.68--0.81).
All 18 exceptions occur at $\delta=0.3$, 15 of them with an early refresh at step 16, exceeding the bound by 1.5\,dB in the median.
We treat this as an empirical regularity over the tested range rather than a general bound. The exceptions point to an interaction between a refresh and the Taylor extrapolation across it.

\begin{figure}[t]
  \centering
  \includegraphics[width=\columnwidth]{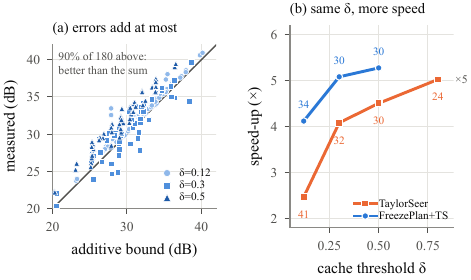}
  \caption{Composition on the development set. (a) Measured composed PSNR vs.\ the PSNR implied by summing the individual MSEs; 90\% of points lie above the diagonal. (b) Speed-up vs.\ cache threshold $\delta$; labels give PSNR. With freezing, 5$\times$ is reached at $\delta=0.3$.}
  \label{fig:mech}
\end{figure}

\subsection{\method{}}
\label{sec:freezeplan}
Fig.~\ref{fig:method} summarizes the schedule and one compact step.
\textbf{Decision.} The first $\tau=4$ steps run in full. For each token we average $\|\xh-\mathrm{ref}\|$ over them; tokens above the median plus $2\times$MAD become active and the active set is dilated by one token. If the median distance exceeds 0.5 (a global edit), nothing is frozen (\emph{fallback}).
\textbf{Freezing.} Frozen tokens follow Prop.~\ref{prop:anchor} toward $\xh(\sigma_{\tau-1})$; their K/V is appended once to the prefix cache.
\textbf{Refresh.} One full step at $r=21$ renews the anchors and the appended K/V; it is optional (Sec.~\ref{sec:ablation}).
\textbf{Caching.} Active tokens use DBCache (F1B0) with first-order TaylorSeer at $\delta=0.3$; our implementation reproduces the official cache-dit~\cite{cachedit} outputs bit-for-bit.
\textbf{Execution.} Active tokens are gathered into a compact sequence, so the transformer processes only them.
All of $\tau$, $r$, the selection rule and $\delta$ were chosen on the development set. The final choice, the composed configuration with the highest PSNR among those measured at $\ge$5$\times$, was fixed before test split B was generated.

\section{Experiments}
\label{sec:exp}

\subsection{Setup}
We use Qwen-Image-2.1~\cite{qwenimage21} (40 steps, CFG$=$1, 1024 px, bf16) on one RTX~4090 and the English part of GEdit-Bench~\cite{step1x} (11 edit types).
Within each type, samples are ordered by the SHA-1 of their key and fixed positions form disjoint splits of 44 samples: development (all design choices), test A and test B (main).
Configurations and pass criteria of each test split were written to a hashed contract before any output was generated; one earlier test run with a different selection rule was stopped at 12/43 samples because its configuration could not reach 5$\times$ (Appendix~\ref{app:protocol}).
We report PSNR and SSIM~\cite{ssim} to the 40-step output (also on truly changed tokens), the GEdit VIEScore prompts~\cite{viescore,step1x} judged by Qwen3-VL-8B~\cite{qwen3vl} as a paired difference to the dense output with bootstrap 95\% CIs ($\Delta$VLM), and the median real end-to-end speed-up measured with CUDA events.
Baselines are TaylorSeer and DBCache with official cache-dit settings ($\delta=0.3$--$1.0$), fewer steps (8/7/6), and SpotEdit and RegionE ported from their official code with TaylorSeer as the step cache; every baseline is pushed to the highest speed it reaches.
Each sample is generated once with a fixed seed (20261023), shared by all methods.
All runs use one RTX~4090 (48\,GB); the test split takes about 2.3\,GPU-hours for 24 configurations and 1.4 more for VLM judging.

\begin{table}[t]
  \centering
  \caption{Test split B (44 samples, pre-registered). Medians; $\Delta$VLM: mean paired difference to the dense output with 95\% CI. $^\dagger$Fastest setting the method reaches.}
  \label{tab:main}
  \footnotesize
  \setlength{\tabcolsep}{2.2pt}
  \resizebox{\columnwidth}{!}{%
  \begin{tabular}{lccccc}
    \toprule
    Method & Speed$\uparrow$ & PSNR$\uparrow$ & Edit$\uparrow$ & SSIM$\uparrow$ & $\Delta$VLM \\
    \midrule
    TaylorSeer $\delta{=}0.8$ & 5.02$\times$ & 23.45 & 19.71 & 0.707 & $-$1.59 {\scriptsize[$-$1.93,$-$1.25]}\\
    TaylorSeer $\delta{=}1.0$ & 5.66$\times$ & 19.39 & 17.11 & 0.522 & $-$4.11 {\scriptsize[$-$4.67,$-$3.53]}\\
    DBCache $\delta{=}0.8$ & 5.03$\times$ & 23.33 & 19.65 & 0.700 & $-$1.67 {\scriptsize[$-$2.01,$-$1.33]}\\
    6 steps & 5.17$\times$ & 22.22 & 17.51 & 0.869 & $-$0.23 {\scriptsize[$-$0.38,$-$0.10]}\\
    SpotEdit+TS$^\dagger$ & 3.90$\times$ & 26.14 & 20.00 & 0.884 & $-$0.44 {\scriptsize[$-$0.78,$-$0.17]}\\
    RegionE+TS$^\dagger$ & 2.70$\times$ & 32.93 & 27.60 & 0.960 & $+$0.12 {\scriptsize[$-$0.10,$+$0.47]}\\
    \midrule
    \method{} & 5.05$\times$ & \tsb{29.77} & \tsb{24.61} & \tsb{0.930} & \tsb{$-$0.06} {\scriptsize[$-$0.15,$+$0.01]}\\
    \method{}, no refresh & 5.53$\times$ & 29.56 & 24.44 & 0.917 & $-$0.11 {\scriptsize[$-$0.20,$-$0.02]}\\
    \quad + $\delta{=}0.5$ & \tsb{5.99$\times$} & 28.02 & 22.24 & 0.893 & $-$0.12 {\scriptsize[$-$0.21,$-$0.02]}\\
    \bottomrule
  \end{tabular}}
\end{table}

\begin{figure}[t]
  \centering
  \includegraphics[width=\columnwidth]{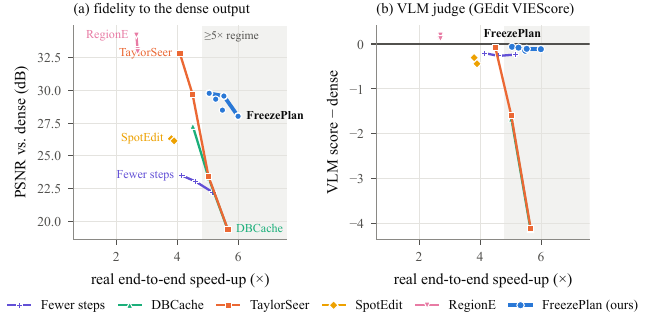}
  \caption{Speed--fidelity on test split B. (a) PSNR to the dense output vs.\ real speed-up; (b) $\Delta$VLM. Shaded: the $\ge$5$\times$ regime.}
  \label{fig:curves}
\end{figure}

\subsection{Main results}
Table~\ref{tab:main} and Fig.~\ref{fig:curves} test whether \method{} outperforms the methods able to reach 5$\times$.
At nearly equal speed (5.05$\times$ vs.\ 5.02$\times$), \method{} is 6.3\,dB above TaylorSeer and keeps the edited region 4.9\,dB higher; against the TaylorSeer curve interpolated at 5.53$\times$ the margin is 9.3\,dB.
$\Delta$VLM shows no significant drop for \method{} ($-$0.06, CI $[-0.15,+0.01]$); this is not a proof of equivalence.
DBCache and fewer steps give the same picture in PSNR.
Neither region method reaches 5$\times$ with its original schedule (Table~\ref{tab:caps}): RegionE+TS stops at 2.70$\times$ with higher fidelity than \method{} (32.9\,dB), and SpotEdit+TS stops at 3.90$\times$ with lower fidelity (26.1\,dB).
On test split A, an earlier version of \method{} (reference anchor, no refresh, DBCache) was 7.8\,dB above TaylorSeer at 5.56$\times$ with $\Delta$VLM $-0.21$; this supports the same direction on independent samples but does not validate the final configuration.

\subsection{Schedule-matched selectors and ablations}
\label{sec:sched}
\label{sec:ablation}
To separate the schedule from the region selector, we kept the schedule and cache of \method{} fixed and swapped only the selector (Table~\ref{tab:abl}, top; a post-hoc comparison, not pre-registered).
All three selectors exceed 5$\times$. Their median PSNR is 0.5--0.6\,dB below that of \method{}, and the paired per-sample difference is 0.26\,dB (95\% CI 0.08--0.46) against RegionE and 0.29\,dB (0.08--0.73) against SpotEdit. RegionE's partition freezes more tokens and runs faster, so it lies close to the same speed--fidelity curve.
$\Delta$VLM differences between selectors are not significant.
The gain of \method{} therefore comes mainly from the schedule and budget split, not from the selector.
Removing one component at a time (Table~\ref{tab:abl}, bottom) shows that the $x_0$ anchor matters most; the refresh is a speed--fidelity knob rather than a requirement: dropping it trades 0.21\,dB for 9.5\% more speed.

\begin{table}[t]
  \centering
  \caption{Top: region selectors under the schedule and cache of \method{} (post hoc). Bottom: one component removed at a time. Test split B; $\rho$: mean fraction of tokens computed per step.}
  \label{tab:abl}
  \setlength{\tabcolsep}{3.5pt}
  \begin{tabular}{lcccc}
    \toprule
    Variant & Speed & PSNR & $\rho$ & $\Delta$VLM\\
    \midrule
    \method{} selector & 5.03$\times$ & 29.77 & 0.41 & $-$0.06\\
    RegionE selector~\cite{regione} & 5.23$\times$ & 29.29 & 0.31 & $+$0.05\\
    SpotEdit selector~\cite{spotedit} & 5.01$\times$ & 29.14 & 0.43 & $-$0.16\\
    \midrule
    reference anchor (SpotEdit-style) & 5.06$\times$ & 27.84 & 0.41 & $-$0.40\\
    no frozen K/V & 5.15$\times$ & 29.02 & 0.41 & $-$0.21\\
    no global fallback & 5.11$\times$ & 28.84 & 0.37 & $-$0.15\\
    random active set & 5.07$\times$ & 28.65 & 0.41 & $-$0.18\\
    \bottomrule
  \end{tabular}
\end{table}

\subsection{Relation to fewer steps}
\label{sec:fewsteps}
Six steps reach 5.17$\times$ at 22.2\,dB but lose only 0.23 VLM points: a shorter trajectory yields a different yet plausible image.
We target reproducing the output the original model produces for a given input and seed, on which \method{} is 7.4\,dB better; it is also slightly better in $\Delta$VLM ($+$0.17, CI $[+0.05,+0.30]$).
The two approaches are orthogonal and can be combined.

\subsection{Qualitative results}
\label{sec:qual}
Fig.~\ref{fig:qual} shows an edit on an original anime illustration that is not part of GEdit-Bench. We generated five such illustrations with Qwen-Image-2.1 and edited them with the configurations of split B; the figure shows the one with the smallest gain of \method{} over the speed-matched baselines. Flat colour regions make step-caching artefacts easy to see: at 5$\times$, TaylorSeer and DBCache turn the cake into noise, while \method{} matches the dense output. Appendix~\ref{app:qual} shows all five illustrations and split-B samples chosen at fixed percentiles of the PSNR gain.

\begin{figure*}[t]
  \centering
  \includegraphics[width=\textwidth]{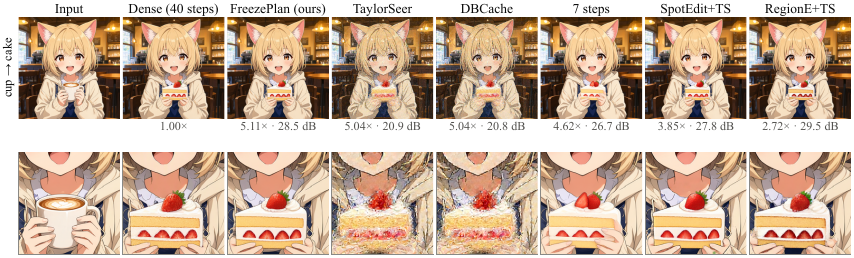}
  \caption{``Replace the cup in her hands with a slice of strawberry cake'' on an original anime illustration (outside the benchmark), at about 5$\times$ real speed-up. Labels: speed-up $\cdot$ PSNR to the 40-step output; bottom row: crop of the edited region. Step caching alone (TaylorSeer, DBCache) breaks down; fewer steps give a plausible but different cake; the re-implemented region methods (SpotEdit, RegionE, both with TaylorSeer) cannot reach this speed with their schedules. \method{} (5.11$\times$) matches the dense output.}
  \label{fig:qual}
\end{figure*}

\section{Limitations}
\label{sec:limits}
(1)~The fallback fired on 11/44 test samples (25\%), where \method{} is slightly below TaylorSeer at the same threshold (details and a fix in Appendix~\ref{app:fallback}).
(2)~Newly added small objects occasionally show faint ghosting near the frozen boundary; we have not quantified it.
(3)~All results are on Qwen-Image-2.1 and rely on its prefix KV cache; the cost coefficients are specific to one GPU.
(4)~The $x_0$ anchor lets unedited regions follow the dense model's slight re-rendering instead of copying the input.
(5)~Each test split has 44 samples generated with one seed, the VLM judge is a local 8B model rather than a commercial one, and the selector comparison of Sec.~\ref{sec:sched} was added after the test results were known.

\section{Conclusion}
\label{sec:conclusion}
On Qwen-Image-2.1, the region methods we re-implemented are capped by the full steps in their schedules, and step caching alone crosses its error knee at high speed-ups.
Spending few full steps and handing the rest of the budget to a conservative cache, \method{} reaches 5.05$\times$ at 29.8\,dB on a pre-registered split with no significant VLM drop, and the schedule-matched comparison locates the gain in the schedule and budget split.

\bibliographystyle{IEEEtran}
\bibliography{refs}

\appendices
\section{Conventions}
\label{app:conventions}
Data labels: \textbf{[dev]} development split, \textbf{[A]} test split A, \textbf{[B]} test split B (main, pre-registered), \textbf{[A6]} post-hoc selector comparison on split B. All numbers are medians over samples unless stated; PSNR and SSIM are measured against the 40-step output of the same run; identical outputs (infinite PSNR) are counted as 100\,dB.

\section{Scope of Proposition 1}
\label{app:scope}
For a token frozen at step $k_0$ with a fixed anchor $c$, $x_k-c=(\sigma_k/\sigma_{k_0})(x_{k_0}-c)$ for all $k\ge k_0$.
The frozen trajectory is therefore fixed by its start, its anchor and the noise schedule; computation and caching of active tokens do not enter it.
The token reaches $c$ only if the anchor is held until $\sigma_K=0$.
At a refresh step $r$ the token is still at $c+(\sigma_r/\sigma_{k_0})(x_{k_0}-c)$; the full step there computes a new anchor $\xh(\sigma_r)$, which \method{} holds to the end.
That anchor is computed on a trajectory shaped by the earlier cached steps, so the final frozen latent depends on the cache before $r$.
Frozen K/V also enter the attention of active tokens, and the VAE decoder mixes neighbouring latents, so pixel errors are not strictly separable.
Sec.~\ref{sec:anchor} claims only that between refreshes the frozen latents do not depend on the concurrent cache.

\section{Exceptions to error additivity [dev]}
\label{app:additivity}
We measured 180 cases: 30 samples in which tokens were frozen, refresh at step 16 or 21, and $\delta\in\{0.12,0.3,0.5\}$.
For each case we compare the composed MSE with the sum of the MSE of freezing alone and of TaylorSeer alone at the same $\delta$.

\begin{table}[h]
  \centering
  \caption{Composed MSE relative to the additive bound (180 cases). CIs bootstrap over samples.}
  \label{tab:exc}
  \setlength{\tabcolsep}{3pt}
  \begin{tabular}{lc}
    \toprule
    Share of cases below the bound & 90\% (95\% CI 86--94\%)\\
    Median ratio composed / bound & 0.76 (95\% CI 0.68--0.81)\\
    Exceptions & 18 (all at $\delta=0.3$)\\
    \quad with refresh at step 16 / 21 & 15 / 3\\
    \quad distinct samples & 15 of 30\\
    Excess over the bound (median / max.) & 1.5 / 4.0\,dB\\
    Cache share of the bound (exceptions / others) & 0.56 / 0.40\\
    \bottomrule
  \end{tabular}
\end{table}

\noindent The exceptions concentrate where the cache error is relatively large and the refresh is early.
A plausible cause is that the Taylor derivative straddles the refresh, mixing features computed with stale frozen K/V and with refreshed ones.
We have not verified this mechanism; a direct test would reset the TaylorSeer history at the refresh.
The later refresh ($r=21$) used by \method{} is less affected.

\section{Cost model}
\label{app:cost}
Eq.~\eqref{eq:cost} was fitted by least squares on 968 development runs ($T_\mathrm{fix}=1.48$\,s, $C=0.575$\,s, $o=0.005$\,s, $c=0.012$\,s; $R^2=0.997$; median absolute time error 2.2\%, 90th percentile 6.4\%).
Table~\ref{tab:oos} applies the same coefficients to test split B without refitting.

\begin{table}[h]
  \centering
  \caption{Out-of-sample prediction on split B (1056 runs). Configurations marked * do not appear in the development set.}
  \label{tab:oos}
  \setlength{\tabcolsep}{3pt}
  \resizebox{\columnwidth}{!}{%
  \begin{tabular}{lccc}
    \toprule
    Configuration & Measured & Predicted & Time error\\
    \midrule
    TaylorSeer $\delta=0.3$ / 0.8 / 1.0 & 4.09 / 5.02 / 5.66$\times$ & 4.15 / 5.14 / 5.84$\times$ & 1.3--2.3\%\\
    8 / 7 / 6 steps* & 4.15 / 4.61 / 5.17$\times$ & 4.03 / 4.45 / 4.97$\times$ & 3.4--4.5\%\\
    RegionE+TS* & 2.66--2.70$\times$ & 2.61--2.64$\times$ & 2.5\%\\
    SpotEdit+TS* & 3.81--3.90$\times$ & 4.09--4.16$\times$ & 6.4\%\\
    \method{}+TS (all variants) & 5.05--5.99$\times$ & 4.78--5.59$\times$ & 4.9--6.0\%\\
    \midrule
    All 1056 runs & & & median 4.0\%, P90 7.9\%\\
    \bottomrule
  \end{tabular}}
\end{table}

\noindent Two biases are systematic and explicable.
The model under-predicts the speed of composed configurations by about 6\%, because the cost of a compact step does not scale exactly linearly with the active fraction.
It over-predicts SpotEdit by about 7\%, because the decode needed by its perceptual selector is not modelled.
All runs used one RTX~4090; the coefficients are hardware-specific, but the structural conclusion (full steps set the ceiling) is not.

\section{Refresh as a speed--fidelity knob}
\label{app:refresh}
On split B, \method{} with a refresh reaches 5.05$\times$ at 29.77\,dB and without it 5.53$\times$ at 29.56\,dB; both lie on the speed--fidelity front.
On the development set the refresh adds 2.1\,dB when no step cache is used (35.1 vs.\ 33.0\,dB), so its benefit shrinks once the cache error dominates.

\section{Alternatives that did not work [dev]}
\label{app:alternatives}
\begin{itemize}
  \item \emph{Per-band anchor mixing:} the optimal weights on the $x_0$ anchor are 1.02--1.16 in every frequency band, and the gain over the plain $x_0$ anchor is at most 0.7\,dB.
  \item \emph{Extrapolating $\xh(\sigma)$ to $\sigma=0$} with polynomials over the warm-up steps is 4--56\,dB worse than the reference anchor; the $x_0$ trajectory drifts before it converges.
  \item \emph{Partial refresh:} the error of the $x_0$ anchor is spread out (the worst 10\% of tokens carry 17\% of it), and decision-time predictors correlate weakly with it (Spearman $\le 0.38$). Refreshing 30\% of tokens recovers 1.0\,dB versus 3.8\,dB for a full refresh.
\end{itemize}

\section{Why the re-implemented region methods stay below \texorpdfstring{5$\times$}{5x}}
\label{app:regionmethods}
\textbf{RegionE.} Its schedule keeps 13 full steps (9 warm-up, 1 forced update, 3 smoothing steps; rescaled from its 28-step setting). The cost model caps it at 2.73$\times$; we measure 2.70$\times$ with TaylorSeer. With its own AVDCache, raising the cache threshold from 0.03 to 0.1 or 0.2 produced identical outputs: the official rule recomputes every step at which the decay ratio $\gamma(1-\Delta\sigma)$ is at least 1, independently of the threshold.

\textbf{SpotEdit.} Its schedule keeps 7 full steps (4 initial steps and resets at 10, 18, 25) and decodes at each decision for its LPIPS selector; the ceiling is 4.45$\times$ and we measure 3.90$\times$.

\textbf{Schedule-matched comparison [A6].} Under the schedule and cache of \method{}, both selectors exceed 5$\times$ (Table~\ref{tab:abl}). Their $\Delta$VLM to the dense output is $+0.05$ $[-0.17,+0.41]$ (RegionE) and $-0.16$ $[-0.32,-0.02]$ (SpotEdit), against $-0.06$ $[-0.15,+0.01]$ for \method{}; the paired differences to \method{} are $-0.11$ $[-0.46,+0.09]$ and $+0.09$ $[-0.01,+0.22]$, neither significant. A reference anchor under the same schedule costs 1.9\,dB (Table~\ref{tab:abl}, bottom).

\section{Protocol, amendments and the stopped run}
\label{app:protocol}
\textbf{Contracts.} For each test split, sample list, configurations and pass criteria were written to a JSON contract whose SHA-256 was recorded before generation.

\textbf{Amendments,} all recorded before the affected analysis:
A1, infinite PSNR is counted as 100\,dB in medians;
A2, comparisons against SpotEdit points use the nearest faster measured point instead of interpolation;
A3, competitor coverage was extended on split A (TaylorSeer above 4.4$\times$, fewer steps);
A5, on split B a local-edit-only analysis was added after 6 of 44 samples, as secondary to the registered all-sample gates;
C1, the development analysis script dropped infinite PSNR values contrary to A1 and was corrected before the automatic configuration choice, which did not change.

\textbf{Stopped run (GEdit positions 12--15).} Its configuration was selected by a rule that ranked region-only quality and chose a decision step $\tau=6$. With seven non-cacheable steps, its composition with TaylorSeer peaked near 4.4$\times$ on local edits, so the $\ge$5$\times$ regime could not be tested, and the run was stopped at 12 of 43 samples. Partial results on its local edits: region-only 2.14$\times$ at 38.6\,dB; with TaylorSeer $\delta=0.8$, 4.37$\times$ at 26.5\,dB versus 5.04$\times$ at 23.9\,dB for TaylorSeer alone; RegionE 2.42$\times$ at 38.8\,dB and, with AVDCache threshold 0.1 or 0.2, an identical 2.52$\times$ at 35.3\,dB.

\textbf{Reproducibility.} The in-path TaylorSeer equals the official cache-dit output bit-for-bit at four thresholds on all 44 samples of split A. On Qwen-Image-2.1, cache-dit's calibrator counter runs two steps behind the denoising step, because the step-0 reference-extraction pass enters the cached blocks; we reproduce this offset. Re-running the dense model on split A reproduced 39 of 44 outputs exactly; the other five differ at 68--70\,dB PSNR (at most 2 grey levels), which we attribute to non-deterministic GPU kernels.

\textbf{RegionE port.} AVDCache factors $\gamma$ are published for 28 steps only; we refit them on the development trajectories as mean norm ratios divided by $(1-\Delta\sigma)$. Velocity reuse follows the official code (last computed velocity times the current ratio), which differs from the cumulative product written in the paper.

\section{Global edits and the fallback [B]}
\label{app:fallback}
The fallback fired on 11/44 split-B samples (25\%), mostly style and background changes. It depends on the measured $x_0$ distance, not on the edit category: motion and portrait edits, although ``global'' in GEdit-Bench, never triggered it and froze about 63\% of tokens.

\begin{table}[h]
  \centering
  \caption{Fallback rate and active fraction on split B.}
  \label{tab:fb}
  \setlength{\tabcolsep}{3pt}
  \resizebox{\columnwidth}{!}{%
  \begin{tabular}{lccc}
    \toprule
    Subset & Samples & Fallback & $\rho$ (no fallback)\\
    \midrule
    Global categories & 20 & 9 (style 4/4, background 4/4, tone 1/4) & 0.37\\
    Local categories & 24 & 2 (colour 1, text 1) & 0.40\\
    All & 44 & 11 (25\%) & 0.38\\
    \bottomrule
  \end{tabular}}
\end{table}

\noindent On the 11 fallback samples, \method{} ($\delta=0.3$) reaches 4.07$\times$ at 25.1\,dB, against 4.09$\times$ at 26.4\,dB for TaylorSeer at the same threshold and 5.02$\times$ at 19.8\,dB at a similar speed. With nothing frozen, \method{} differs from TaylorSeer only by the refresh step, so the 1.3\,dB gap most likely comes from the refresh interrupting the Taylor extrapolation; the exceptions to error additivity (Appendix~\ref{app:additivity}; 15 of 18 with a refresh at step 16) point the same way. Skipping the refresh when the fallback fires is a straightforward fix that we have not evaluated, since it would change the pre-registered configuration. On the 11 global-category samples without fallback, \method{} reaches 5.13$\times$ at 31.1\,dB against 5.01$\times$ at 24.1\,dB for TaylorSeer $\delta=0.8$. By category group: local, \method{} 5.07$\times$/31.1\,dB ($\Delta$VLM $-0.02$) vs.\ TaylorSeer 5.02$\times$/24.4\,dB; global, 4.96$\times$/28.7\,dB vs.\ 5.02$\times$/22.5\,dB.

\section{Fewer steps on both test splits}
\label{app:fewsteps}
On split A (early \method{} version), 10/8/7/6 steps give 3.47/4.16/4.61/5.18$\times$ at 26.2/25.3/24.6/23.7\,dB with $\Delta$VLM $+0.02$/$-0.02$/$-0.10$/$-0.17$. On split B, 8/7/6 steps give 4.15/4.61/5.17$\times$ at 23.5/23.0/22.2\,dB with $\Delta$VLM $-0.21$/$-0.27$/$-0.23$. At a similar speed, six steps lose 4.3\,dB to the early \method{} version on split A and 7.4\,dB to the final version on split B in fidelity to the original output, while staying close to the original in the VLM judge.

\section{Qualitative results}
\label{app:qual}
Figure~\ref{fig:app-qual} shows split-B samples chosen by a fixed rule: among local edits with freezing, the samples at the 25th, 50th and 75th percentile of the PSNR gain of \method{} over the speed-matched baselines, plus the median global-edit sample. Figure~\ref{fig:app-anime} shows all five original anime illustrations, generated with Qwen-Image-2.1 from text descriptions of original characters and edited with the configurations of split B; they are not part of any benchmark. Over the speed-matched baselines, \method{} gains 7.6--10.3\,dB on them.

\textbf{Ghosting.} In the split-B ``add a tennis ball'' sample (second row of Fig.~\ref{fig:app-qual}), \method{} shows a faint second ball near the frozen boundary. We have not quantified how often this occurs.

\begin{figure*}[t]
  \centering
  \includegraphics[width=0.94\textwidth]{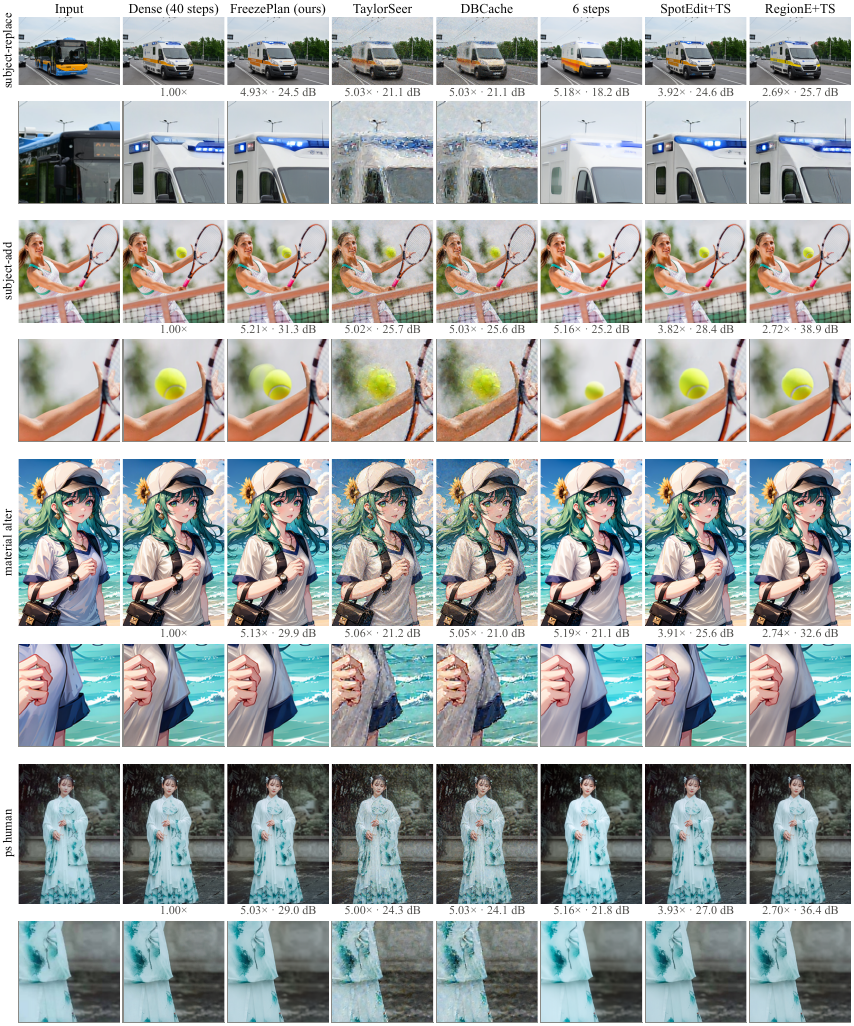}
  \caption{Split B, about 5$\times$; rows chosen by the rule above (crops of the edited region below each row). Labels: speed-up $\cdot$ PSNR to the 40-step output.}
  \label{fig:app-qual}
\end{figure*}

\begin{figure*}[t]
  \centering
  \includegraphics[width=0.86\textwidth]{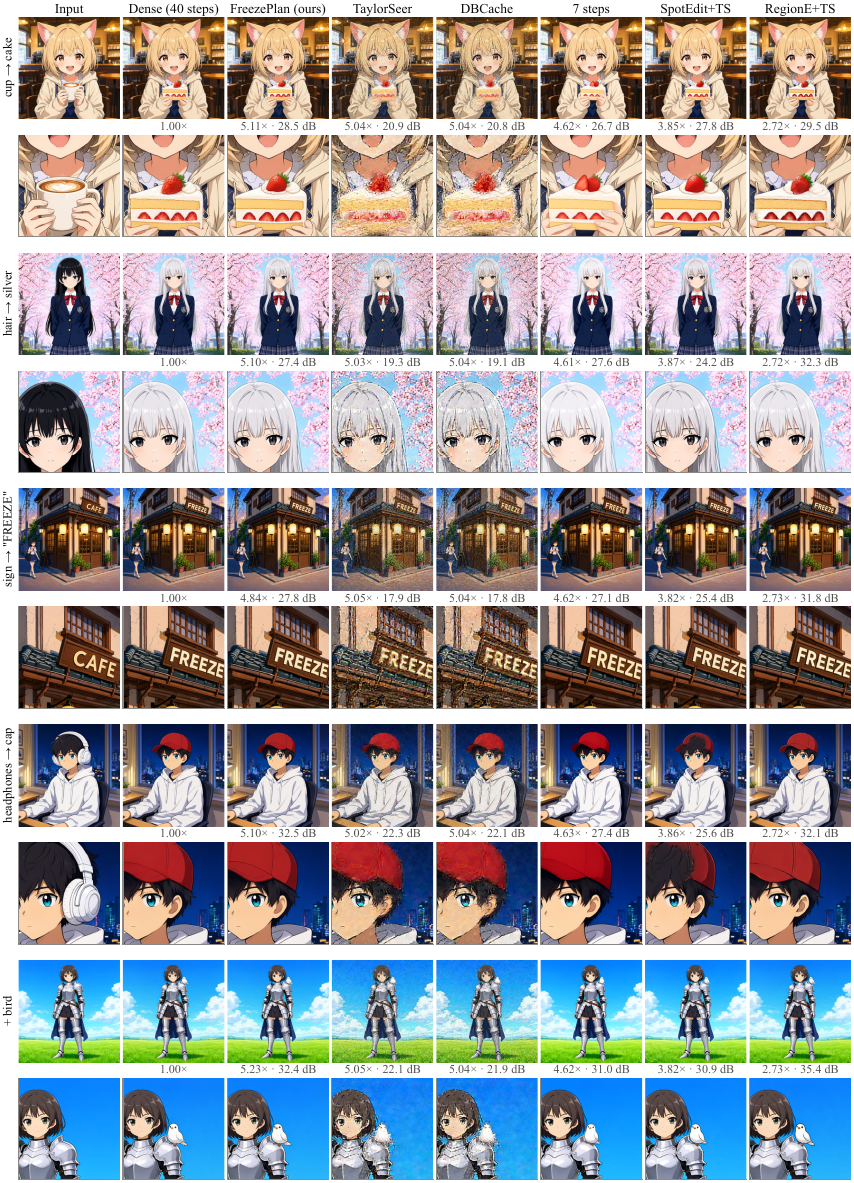}
  \caption{All five original anime illustrations, about 5$\times$, same configurations as split B.}
  \label{fig:app-anime}
\end{figure*}

\section{Relation to HyCa, BudCache and ProCache}
\label{app:related}
HyCa (ICLR 2026) assigns different extrapolation solvers to different feature dimensions and reports 6.24$\times$ on Qwen-Image and Qwen-Image-Edit. It improves the step-cache operator and could replace TaylorSeer inside \method{}. We did not port it to the prefix KV cache of Qwen-Image-2.1; its reported numbers use other models, quality criteria and samples and are not directly comparable to ours.
BudCache (ICML 2026) searches step-level caching policies under a fixed number of function evaluations (up to about 3.7$\times$ on FLUX.1 and Wan2.1), and ProCache (AAAI 2026) searches constrained caching schedules with selective token computation (up to about 2.9$\times$). Both allocate the budget along the step axis. We study how spatial freezing shifts the operating point of that axis, so these schedulers could serve as the temporal component of \method{}.
\end{document}